%% file: main.tex
\documentclass[dvipsnames]{article} %
\usepackage{colm2024_conference}

\usepackage{booktabs}
\usepackage{enumitem}
\usepackage{wrapfig}
\usepackage{algorithm}
\usepackage{algpseudocode}
\usepackage{graphicx}
\usepackage[misc]{ifsym}
\usepackage{amsmath}
\usepackage{amsthm}
\usepackage{amssymb}
\usepackage{microtype}
\usepackage{amsmath}
\usepackage{colortbl}
\usepackage[utf8]{inputenc}
\usepackage[T1]{fontenc}
\definecolor{lightgray}{rgb}{0.9,0.9,0.9}
\usepackage{caption}
\usepackage{subcaption}
\usepackage{setspace}
\usepackage{url}
\usepackage{multirow}
\usepackage{tabularx}
\usepackage{blindtext}
\usepackage{pgfplots}
\pgfplotsset{compat=1.18} 
\usepackage{tikz}
\usetikzlibrary{er,positioning,bayesnet}
\usepackage{makecell}
\usepackage{tipa}
\usepackage{siunitx}
\usepackage{nicefrac}
\usepackage{listings}
\usepackage[raster,skins, most]{tcolorbox} %
\usepackage{xltabular}
\usepackage{adjustbox}
\usepackage{xurl}
\usepackage{rotating}
\usepackage[normalem]{ulem}

\usepackage{hyperref}
\usepackage{url}
\usepackage{amssymb}
\usepackage{graphicx}
\usepackage{enumitem}
\usepackage{wrapfig}
\usepackage{caption}

\usepackage{amsthm}

\usepackage{graphicx}    
\usepackage{subcaption}  

\newtheorem{proposition}{Proposition}

\newcommand{\w}{WebShaper}

\usepackage{fontawesome}

\useunder{\uline}{\ul}{}

\input{math_commands.tex}

\renewcommand{\mslink}{https://modelscope.cn/datasets/iic/WebShaper}
\renewcommand{\ghlink}{https://github.com/Alibaba-NLP/WebAgent}
\renewcommand{\hflink}{https://huggingface.co/datasets/Alibaba-NLP/WebShaper}

\newcommand*\justify{%
  \fontdimen2\font=0.4em
  \fontdimen3\font=0.2em
  \fontdimen4\font=0.1em
  \fontdimen7\font=0.1em
  \hyphenchar\font=`\-
}

\renewcommand{\texttt}[1]{%
  \begingroup
  \ttfamily
  \begingroup\lccode`~=`/\lowercase{\endgroup\def~}{/\discretionary{}{}{}}%
  \begingroup\lccode`~=`[\lowercase{\endgroup\def~}{[\discretionary{}{}{}}%
  \begingroup\lccode`~=`.\lowercase{\endgroup\def~}{.\discretionary{}{}{}}%
  \catcode`/=\active\catcode`[=\active\catcode`.=\active
  \justify\scantokens{#1\noexpand}%
  \endgroup
}

\usepackage{makecell}
\usetikzlibrary{tikzmark}
\makeatletter
\newcommand*\myfontsize{%
  \@setfontsize\myfontsize{7}{8}%
}
\makeatother

\definecolor{uclablue}{RGB}{159, 195, 224}

\definecolor{uclagold}{RGB}{255, 240, 180}

\definecolor{aliceblue}{RGB}{255, 238, 241}

\definecolor{cadmiumgreen}{rgb}{0.0, 0.42, 0.24}

\definecolor{myred}{rgb}{0.7, 0.3, 0.0}
\definecolor{myblue}{rgb}{0.2, 0.3, 0.6}
\definecolor{babygreen}{rgb}{0.85, 0.97, 0.85}

\newcommand{\rcup}{\cup}

\definecolor{purple1}{RGB}{126, 107, 196}
\definecolor{purple2}{RGB}{199, 158, 207}
\definecolor{purple3}{RGB}{214, 200, 255}
\definecolor{purple4}{RGB}{254, 240, 255}
\definecolor{custompurple}{HTML}{491f97}
\definecolor{deepblue}{RGB}{48, 58, 82}
\definecolor{darkred}{HTML}{C00000}

\newcommand{\symboletongyi}{\raisebox{0pt}{~\includegraphics[scale=0.012]{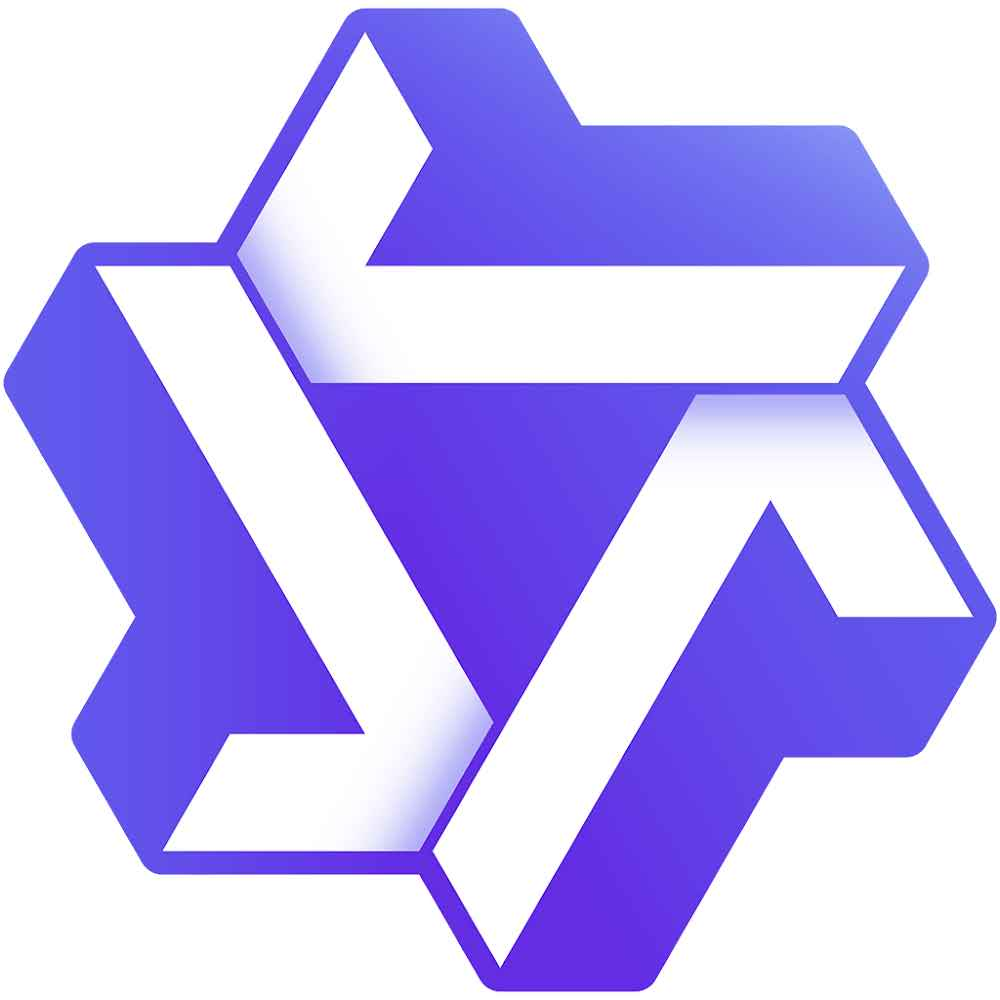}}~}

\definecolor{deepPurple}{HTML}{330066}

\definecolor{uclablue_old}{rgb}{0.15, 0.45, 0.68}
\hypersetup{
    breaklinks,
    citecolor=uclablue_old,
    colorlinks=true,
}

\newtcolorbox{mybox}[2][]
  {colback = black!5!white, colframe = black!75!black, fonttitle = \bfseries,
    colbacktitle = black!100!black, enhanced, before upper={\fontsize{8}{11}\obeyspaces\obeylines\selectfont}, fontupper=\selectfont,
    attach boxed title to top left={yshift=-2.2mm,xshift=4mm},
    title=#2,#1}

\title{\w: Agentically Data Synthesizing via\\ Information-Seeking Formalization}

\author{%
\small{Zhengwei Tao$^{*}$, Jialong Wu$^{*}$, Wenbiao Yin$^{(\textrm{\Letter})}$, Junkai Zhang,  Baixuan Li, Haiyang Shen, Kuan Li, Liwen Zhang, Xinyu Wang, Yong Jiang$^{(\textrm{\Letter})}$, Pengjun Xie, Fei Huang, Jingren Zhou}%
  \\[1em]               
  {\fontsize{10pt}{11pt}\selectfont          
Tongyi Lab\symboletongyi, Alibaba Group}\\
}

\begin{document}

\maketitle

\begingroup
  \renewcommand\thefootnote{*} 
   \footnotetext{denotes equal contribution. $^{\textrm{\Letter}}$~ denotes the correspondence. \{yinwenbiao.ywb, yongjiang.yj\}@alibaba-inc.com} 
\endgroup

\begin{abstract}

The advent of Large Language Model (LLM)-powered agents has revolutionized artificial intelligence by enabling solutions to complex, open-ended tasks through web-based information-seeking (IS) capabilities.
The scarcity of high-quality training data has limited the development of IS agents. 
Existing data synthesis approaches typically adopt an \emph{information-driven} paradigm that first collects web data and then generates questions based on the retrieval.
However, this may lead to inconsistency between information structure and reasoning structure, as well as between the question and the corresponding answer.
To mitigate, we propose a \emph{formalization-driven} IS data synthesis framework \w, which systematically formalizes IS tasks using set-theoretic constructs.
Central to the formalization is the concept of Knowledge Projections (KP), which enables precise control over reasoning structure by KP operation compositions. 
During synthesis, we begin by creating seed tasks, then use a multi-step expansion process.
At each step, an agentic Expander expands the current formal question more complex with retrieval and validation tools based on our formalization.
We train our model on the synthesized dataset.
Experiment results demonstrate that \w~achieves state-of-the-art performance among open-sourced IS agents on GAIA and WebWalkerQA benchmarks. 

\end{abstract}

\begin{figure*}[h]
    \centering
    \includegraphics[width=0.8\linewidth]{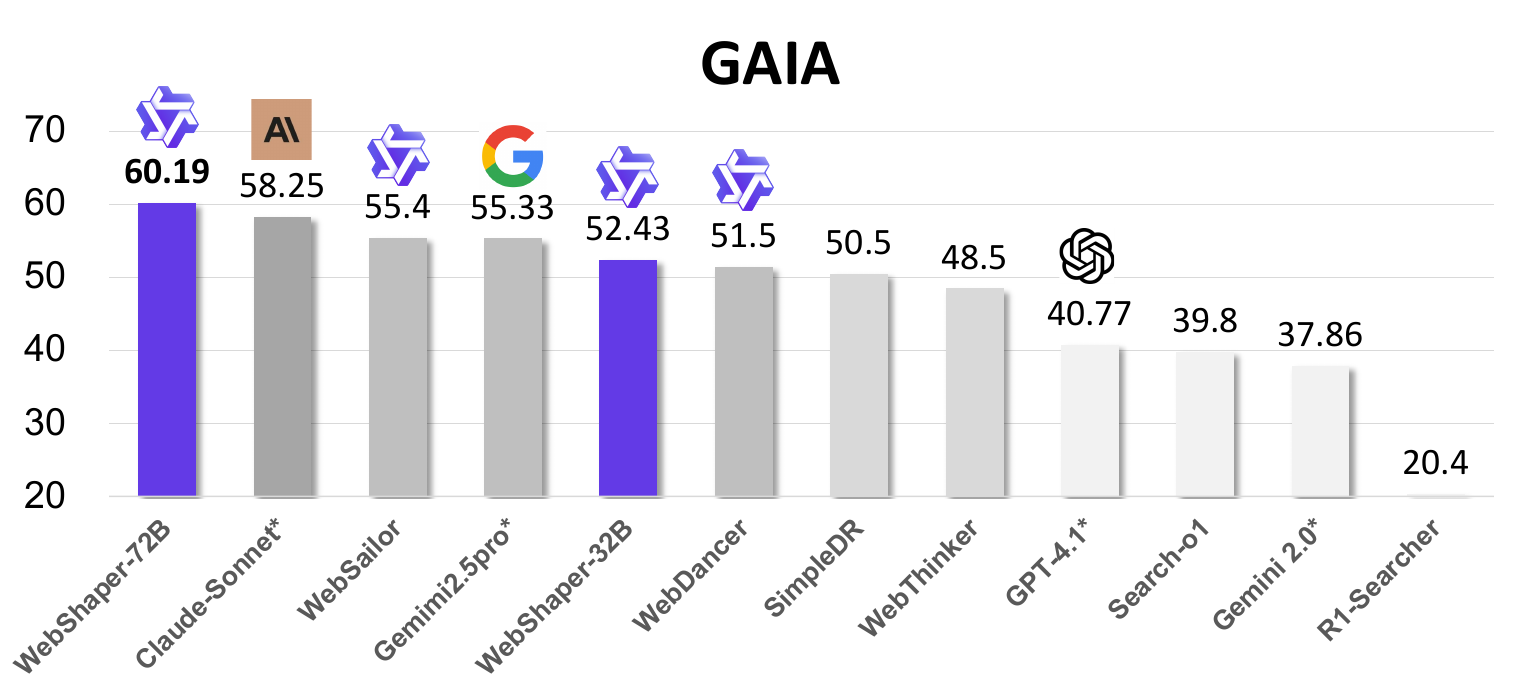}
    \caption{Results on GAIA information-seeking subset among the cutting-edge Deep Research models or systems.
    $^*$ denotes the results using our two browsing tools via function calling APIs.}
    \label{fig: intro}
\end{figure*}



\newpage

\input{sections/1-intro}
\input{sections/2-method}

\input{sections/3-experiments}
\input{sections/4-related-work}

\input{sections/5-conculusion}

\clearpage
\bibliography{biblio}
\bibliographystyle{colm2024_conference}

\clearpage
\appendix
\input{sections/6-appendix}

\end{document}

%% file: math_commands.tex
\usepackage{amsmath,amsfonts,bm}

\def\eqref#1{equation~\ref{#1}}

\def\1{\bm{1}}

\DeclareMathAlphabet{\mathsfit}{\encodingdefault}{\sfdefault}{m}{sl}
\SetMathAlphabet{\mathsfit}{bold}{\encodingdefault}{\sfdefault}{bx}{n}

\def\gE{{\mathcal{E}}}

\def\gT{{\mathcal{T}}}



%% file: sections/1-intro.tex
\section{Introduction}

The emergence of Large Language Model (LLM)-powered language agents has marked a paradigm-shifting advance in artificial intelligence, enabling transformative solutions to previously intractable challenges across domains~\citep{guo2024large, wang2024survey, autogpt, wu2023autogen, mplug-owl}.
Information-seeking (IS) represents a core component of the cognitive autonomy of language agents.
This capability not only underpins their adaptability in open-ended tasks but also powers a range of powerful commercial systems such as Deep Research of OpenAI~\citep{openaidr}, Gemini~\citep{geminidr}, and Perplexity~\citep{perplexity}. 

Current agentic systems for unlocking this capability typically follow a well-established pipeline in agent development: 
(1) First, construct task-specific trajectories of question-answer pairs; 
(2) Employ supervised fine-tuning (SFT) to acquire foundational skills \citep{sun2025simpledeepsearcher}. (3) Generalize strategic decision-making through on-policy reinforcement learning (RL) \citep{jin2025search}.
The entire development of the IS agent originates from and its ultimate effectiveness depends on high-quality IS task training data.
However, due to its complexity, such a high-quality dataset is both sparse and difficult to construct through crowdsourcing. \textbf{Thus, constructing training data through a carefully designed agent pipeline becomes the cornerstone of effective IS agent development.}

\begin{figure*}[h]
    \centering
    \includegraphics[width=0.8\linewidth]{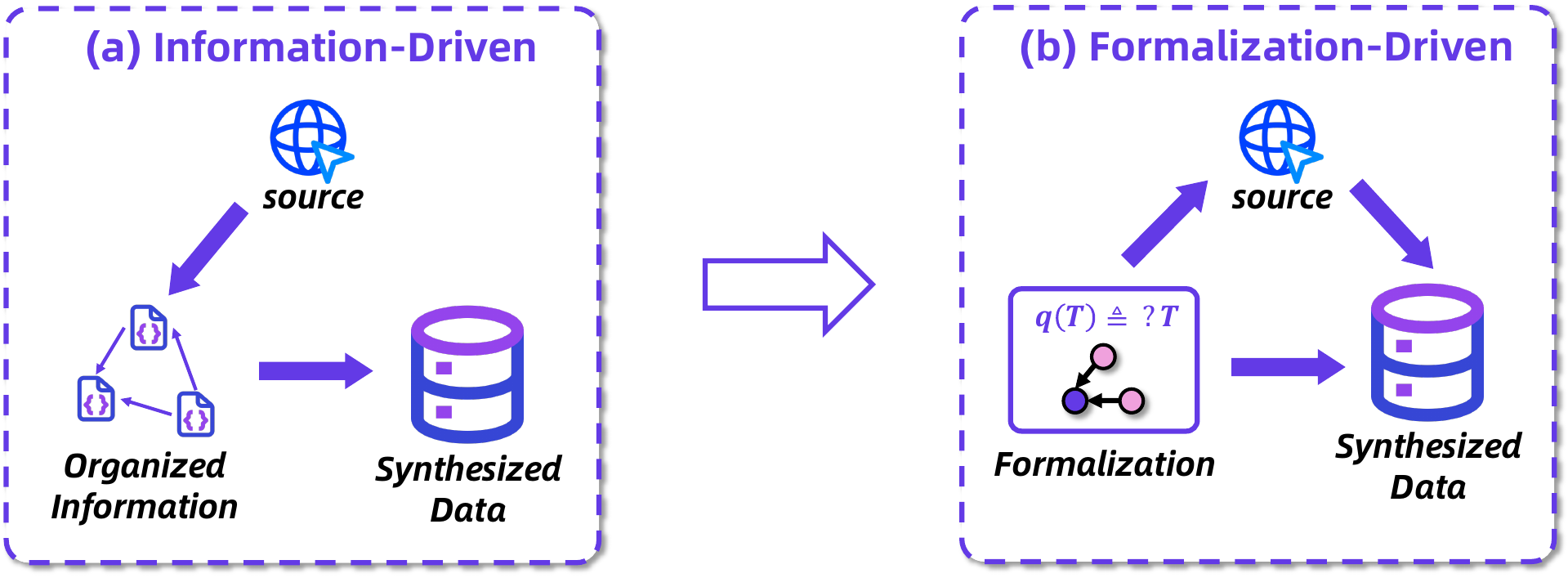}
    \caption{Data synthesis paradigm shift from information-driven to formalization-driven. ``Source'' stands for information sources such as the internet and databases.
    ``Data'' represents the synthesized QA data.
    (a) Previous methods retrieve and organize collected information in advance, then synthesize data according to the information structures.
    (b) Our method establishes the task formalization first, then collects information, and synthesizes QA data based on the formalization.
    }
    \label{fig:intro_example}
\end{figure*}

Existing IS dataset synthesis methods typically involve freely pre-searching for information online and employing LLMs to generate questions from the collected content (Figure~\ref{fig:intro_example}(a)). These approaches first organize the collected information into structured formats, then prompt the LLM with the structured data to produce natural language (NL) questions. Their core objective is to map \textit{\underline{information structures}} into \textit{\underline{reasoning structures}} within the resulting NL questions.
Representative methods like WebDancer~\citep{wu2025webdancer} and TaskCraft~\citep{shi2025taskcraft} generate linear information chains, while others construct graphs connected via web links~\citep{webwalker} or entity coreference networks~\citep{li2025websailor}. 
However, these information-driven approaches face two critical limitations. 
\textbf{First}, the synthesis using LLM may struggle to fully comprehend the information structure, resulting in inconsistent reasoning structures or incorrect answers to the generated NL questions. 
\textbf{Besides}, disordered information retrieval will lead to excessive data processing and will collect redundant homogeneous information structures, which limits the diversity of information structures and reduces knowledge coverage.

To overcome these limitations, we propose \w\footnote{Without loss of generality, we use \w~to denote our data method, dataset, and model.}, a formalization-driven IS data synthesis paradigm, WebShaper, as illustrated in Figure~\ref{fig:intro_example}(b). 
Unlike prior approaches, we first formalize information-seeking tasks and then systematically guide data synthesis through this formalization.
During generation, information collection is explicitly controlled by formal task requirements. This framework offers three key advantages:
\begin{enumerate}
    \item \textit{\textbf{Broader Task Coverage}}: Systematic exploration of task formalizations enables synthesizing diverse information-seeking patterns unconstrained by pre-retrieval content limitations;
    \item \textit{\textbf{Task Controllability}}: Explicit formalization parameters allow precise specification of reasoning structures and complexity levels;
    \item \textit{\textbf{Structural and Answer Consistency}}: Due to the inherent interpretability and verifiability of formalized representations, synthesized outputs exhibit fewer inconsistencies across both information-reasoning structures and question-answer pairs.
\end{enumerate}
\w~ works fundamentally because it introduces a formalization-guided framework that serves as a \textcolor{blue}{\textbf{structural skeleton}} during data synthesis. With this structured guidance, we produce consistent reasoning and redundancy while ensuring rich, diverse reasoning logic.

We leverage the proposed framework to construct the \w~dataset, which serves as training data for the IS agent.
At the core of our framework lies a formalization of IS tasks, which enables principled and systematic generation of task instances with controllable collection complexity and reasoning paths.
This overcomes the fragmented and ad-hoc nature of task construction in prior information-driven approaches.
Unlike relevant fields, where there exists task formalization in advance, such as Lean 4 language~\citep{moura2021lean} in math proving and propositional logic in knowledge-centric question answering~\citep{xia2025improving}, there's no established formalization for information-seeking. 
To the best of our knowledge, we are the first to derive it based on set theory.
\w~treats IS as a unified problem space where task is systematically derived from compositions of basic units termed Knowledge Projections (KP).
To align with the formalized structure, we initiate synthesis by constructing foundational seed tasks, followed by a multi-step expansion grounded in our formal framework. 
This process employs a dedicated agentic Expander module designed to interpret task requirements via KP representations. 
At each expansion stage, the expander transforms the current formal question into a more complicated one.
It implements layer-wise expansion mechanisms that minimize redundancy while preventing reasoning shortcuts through controlled complexity progression.
The Expander operates autonomously during synthesis, performing three core functions: (1) internet-based knowledge collection guided by formal requirements, (2) construction and validation of new formalized problems, and (3) generation of final questions. 
This process ensures a broad coverage of the formalized task space and the correctness of the question and answer.

We conduct extensive experiments to validate \w~dataset by training agents.
Comparison with the existing training dataset shows the effectiveness of \w.
\w~achieves best performances among all open-source IS agents on the GAIA and WebWalkerQA benchmarks. 
Further discussions demonstrate the validity of each module of our method. We summarize our contributions as:

\begin{itemize}
    \item We introduce \w, a formalization-driven data synthesis method for information-seeking agents, grounded in our proposed task formalization. 
    Leveraging this method, we construct the \w~dataset, which enables systematic generation of IS instances.
    \item We propose an agentic Expander that iteratively generates and validates questions in alignment with the formalization.
    \item We conduct extensive experiments across multiple benchmarks to evaluate the effectiveness of WebShaper. Empirical results demonstrate that models trained with \w~consistently outperform baselines, confirming the value of our formalization and synthesis approach.
\end{itemize}

%% file: sections/2-method.tex
\section{Information-Seeking Formalization}


\begin{figure*}
    \centering
    \includegraphics[width=0.8\linewidth]{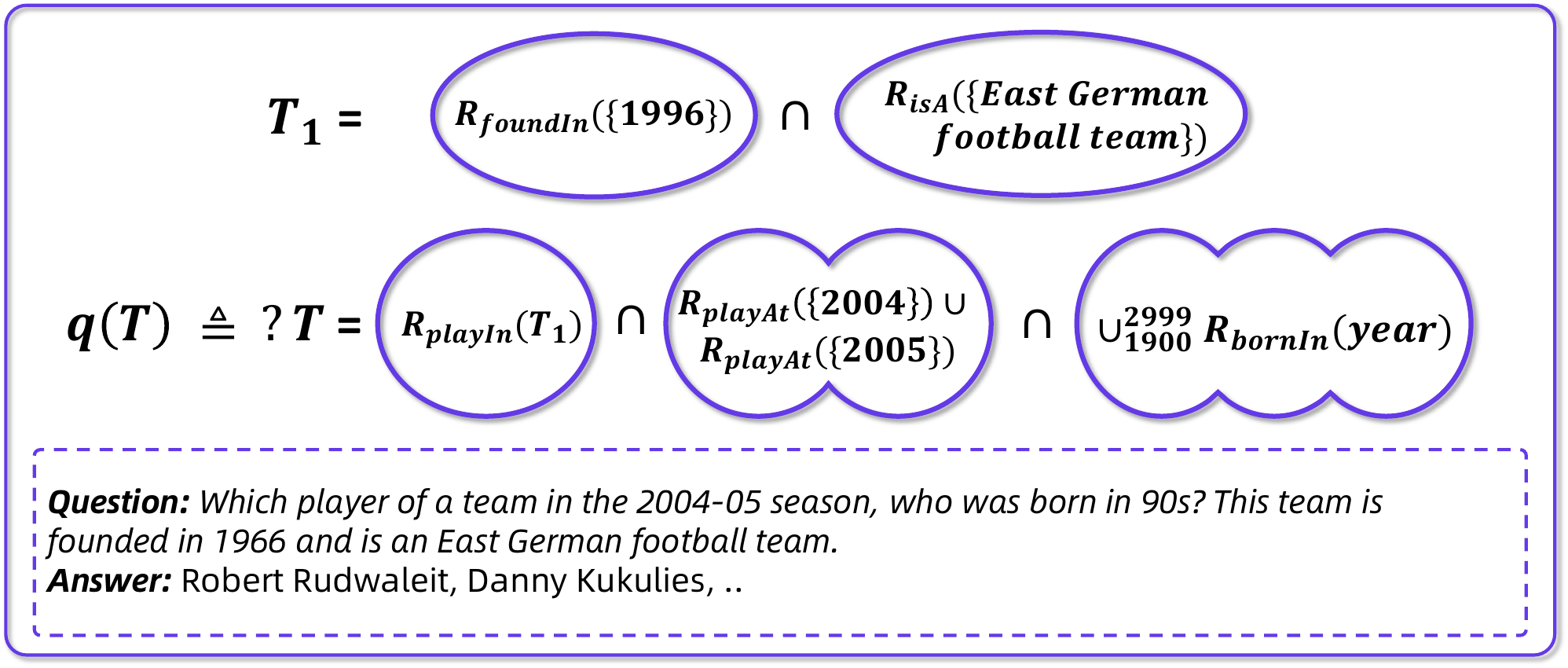}
    \caption{A question-answer case in our information-seeking formalization. We use the \textcolor{custompurple}{\textbf{purple}} diagram to represent a knowledge projection, which is a set of entities.
   }
    \label{fig: example}
\end{figure*}

In this section, we introduce our formalization of the information-seeking task. We illustrate an example in Figure~\ref{fig: example}. An information seeking task $q(T)$ aims to search for knowledge and facts prompted by given facts and locate the answer entity set $T$.
For a basic example also shown in Figure~\ref{fig: example}:

\begin{equation}
\label{eq: eg}
\begin{aligned}
    q(T)= & \textit{Which player of a team in the 2004-05 season, who was born in 90s?} \\
    & \textit{This team is founded in 1966 and is an East German football team.}
\end{aligned}
\end{equation}


To solve it, one should seek information about \textit{This team is founded in 1966 and is an East German football team} to find that the team is \textit{Berliner FC Dynamo}. And then seek for players of \textit{Berliner FC Dynamo team} in 2004 and 2005 respectively and \textit{players born in 90s}, then reason the answer $T=\{\textit{Robert Rudwaleit, Danny Kukulies, ..}\}$.

Let $\gE$ denote the universal set of entities (e.g., players, teams, years). 
Let $R\subseteq \gE \times \gE$ denote a subspace of entity pairs where they have a certain relation.
For example, if the relation is $\textit{bornIn}$, $R$ stands for all pairs of ($\textit{person}$, $\textit{year}$) where $\textit{person}$ is born in $\textit{year}$.

For a subset $V \subseteq \gE$ and a sub-space $R$, define a Knowledge Projection (KP):
\begin{equation}
\label{eq: KP}
R(V) = \{u \mid \exists v \in V,\ (u,v) \in R \text{ or } (v,u) \in R\}.
\end{equation}
For example, when $R$ denotes entity pairs of relation $\textit{bornIn}$, $R(\{\textit{90s}\})$ represents the set of all people born in 90s. \textbf{A KP is \textcolor{blue}{the set of entities} under a certain relation to other entities, which is the basic unit in an information-seeking task.} KP has two operations:

\paragraph{$R$-Union $\rcup$}  
In IS, the question may be seeking for a broader condition due to uncertainty about the target. For instance, we only know the target player was playing between 2000-2010 rather than the exact year in advance. The condition can not be more specific than a year range.

Therefore, given $S_1, S_2$ be entity sets and $R$, then:  
\begin{equation}
\label{eq: union}
R(V) = R(S_1) \rcup R(S_2) \rcup \cdots \rcup R(S_m)
\end{equation}
represents $R(V)$ is the union result set in which the entities have a certain relation to entries in either $S_1$, $S_2$, ..., $S_m$.
If $R$ stands for relation $\textit{playAt}$, then the set of players who play between 2000-2010 is $R(\{\textit{2000}\}) \rcup R(\{\textit{2001}\}) \rcup \cdots \rcup R(\{\textit{2010}\})$.



\paragraph{Intersection $\cap$}  
Some IS tasks require the target to satisfy several conditions simultaneously. It's interpreted as an Intersection operation of KP:

\begin{equation}
\label{eq: intersection}
R(V) = R_1(S_1) \cap R_2(S_2) \cap \cdots \cap R_n(S_n)
\end{equation}
where $R_i$ are about different relations. 
For example, if $R_1$ is about $\textit{playAt}$ and $R_2$ is about $\textit{bornIn}$, then $R_1(\{\textit{2000}\}) \cap R_2(\{\textit{90s}\})$ stands for players playing in $\textit{2000}$ and born in $\textit{90s}$.

Based on $R$-Union and Intersection operations, we introduce IS task formalization. First, we define $T$ as a target set:

\begin{equation}
\label{eq: t}
T =\bigcap_{i=1}^p (R_i(S_{i,1})\cup R_i(S_{i,2})\cup\ldots R_i(S_{i,t_i}))).
\end{equation}

$S_{i,j}\subset \gE$ is an entity set. More generally, $T$ can be recursivelly derived by replacing $S_{i,j}$ with other target set as:

\begin{equation}
\label{eq: t_recurssive}
T=R_1(T_1)\cap R_2(T_2)\cap\ldots\cap R_k(T_k)
\end{equation}

An IS task is to find what entities a questioned $T$ contains:

\begin{equation}
\label{eq: qt}
q(T) \triangleq ?T 
\end{equation}

Therefore, the question example in Eq. (\ref{eq: eg}) can be formalized as:

\begin{equation}
\begin{aligned}
    q(T) \triangleq ?T = & R_{playIn}(T_{1}) \cap (R_{playAt}(\{2004\}) \cup R_{playAt}(\{2005\})) \cap \bigcup_{1900}^{1999} R_{bornIn}(\{y\}))\\
    & T_{1} = R_{foundIn}(\{1996\})\cap R_{isA}(\{East\ German\ football\ team\})
\end{aligned}
\end{equation}

\section{Data Synthesis}

In this section, we describe the process of our data synthesis with our task formalization. 
As Eq.~(\ref{eq: t}-\ref{eq: qt}) shows, an IS task is recursively composited by knowledge projections. In order to better fit the IS task formalization, we start with constructing a seed task, followed by a multi-step expansion approach. This expansion process is built upon our formalization. We then introduce an agentic Expander. It can understand the task formalization with our KP representation. At each expansion step, we implement the layer-wise expansion to reduce redundancy and reasoning shortcuts. 
The Expander autonomously retrieves knowledge from the internet, constructs and validates the new FPs to obtain the new question. We elaborate on this process in the following sections.


\subsection{Seed Question Construction}

The first stage of our data synthesis pipeline involves acquiring a substantial volume of diverse and non-trivial seed questions. To enhance acquisition efficiency, we constructed an offline Wikipedia database by downloading all URLs corresponding to Wikipedia articles while preserving the hyperlinks between them. Subsequently, we perform random walks across these articles through their preserved connections. By aggregating the content from articles traversed during these random walks, we utilize an LLM to generate synthetic data instances. Critically, the generated question-answer pairs must be entirely grounded in the content from the collected articles, without relying on external knowledge sources. 

However, the resulting seed questions could be noisy and contain hallucinations. We launch a filtering process. We complete all the seed questions by WebDancer framework~\citep{wu2025webdancer} based on the QwQ model~\citep{qwq}. We perform 5 times rollouts for each question and keep the data where there must be as least one rollout correctly answering the question. 
We finally construct 18k seed questions. 
We denote the harvested seed question as $q^{1}(T)$.

\subsection{Agentic Expansion}
Subsequently, we progressively expand seed questions into increasingly complex ones through $n$-step expansion $q^{n+1}(T)=\mathrm{Expand}(q^{n}(T))$ guided by the task formalization. However, the IS formalization in Eq.~(\ref{eq: t}-\ref{eq: qt}) is complicated. The nature of recursion and the composition of multiple operations are hard for the model to understand during the synthesis. Besides, since the synthesis relies on retrieving new knowledge online, there are several intermediate processes, such as knowledge filtering and selection. 

Therefore, we establish an Agentic Expansion. We first introduce the KP representation, which enables clear comprehension of our IS formalization. Then, we propose the Layer-wise Expansion Strategy to mitigate the limitations of redundant and reasoning shortcuts. The core of the expansion is the Expander, which is an agent itself to autonomously retrieve information and validate the generation.




\subsubsection{KP Representation}
Since $q(T)$ contains recursion and composition of $R$-Union and Intersection operations, it's not trivial to represent $q(T)$ in the Expander agent prompt.
We introduce our KP Representation. The key to this representation is to: 1) represent a KP unit. 2) can handle $R$-Union and Intersection operations. 3) can handle recursions of KPs.
We start with introducing Constant and Variable:

\begin{itemize}
    \item \textit{Constant}: A constant is a subset of $\gE$ explicitly defined by its elements, e.g., $\{\textit{90s}\}, \{\textit{2004}, \textit{2005}\}$.
    \item \textit{Variable}: A variable is a subset of $\gE$ whose elements are not explicitly given.  
    It may appear as a symbolic placeholder in an expression.
\end{itemize}

Then, we use a triplet [$X$, $r$, $S$] to represent a KP $R(S)$. $r$ is the name of the relation $R$. $X$ is a variable while $S$ can be a variable or a constant. 

We use the prefix $V@$ followed by a variable to denote the variable $V$. We use the prefix $@C$ before its natural language description to represent a constant. For example, $R_{\textit{bornIn}}(\{\textit{90s}\})$ is represented as [$@V$, $\textit{bornIn}$, $\textit{90s}$].
The Intersection operation in Eq.(\ref{eq: intersection}) can be naturally represented as a list of triplets [[$X$, $r_1$, $S_1$], [$X$, $r_2$, $S_2$], ..., [$X$, $r_n$, $S_n$]].

For the $R$-Union in Eq.(\ref{eq: union}), simply expressing it in a list-like form will make the representation complicated in recursive $R$-Union and Intersection. We notice $R$-Union has the following proposition:

\begin{proposition}
\label{eq: union-property}

For a certain $R$, $R$-union satisfies the distributive Law:

\begin{equation}
\label{eq: prop}
R(S_1) \cup R(S_2) = R(S_1 \cup S_2)
\end{equation}

\end{proposition}

\begin{proof}
Let $x$ be an element of $R(S_1) \cup R(S_2)$.  
By Equation~\ref{eq: KP},  
there exists either a $y_1 \in S_1$ such that $(y_1, x) \in R$ or $(x, y_1) \in R$,  
or a $y_2 \in S_2$ such that $(y_2, x) \in R$ or $(x, y_2) \in R$.  
Consequently, there exists a $y \in S_1 \cup S_2$, e.g., $y_1$ or $y_2$,  
such that $(y, x) \in R$ or $(x, y) \in R$.  
Thus, we have $x \in R(S_1 \cup S_2)$,  
and hence $R(S_1) \cup R(S_2) \subseteq R(S_1 \cup S_2)$.

Conversely, let $z$ be an element of $R(S_1 \cup S_2)$.  
Then there exists a $y \in S_1 \cup S_2$ such that $(y, z) \in R$ or $(z, y) \in R$.  
If $y \in S_1$, then $z \in R(S_1)$;  
if $y \in S_2$, then $z \in R(S_2)$.  
In either case, $z \in R(S_1) \cup R(S_2)$.  
Therefore, $R(S_1 \cup S_2) \subseteq R(S_1) \cup R(S_2)$.

Combining both directions, we conclude that:
\[
R(S_1) \cup R(S_2) = R(S_1 \cup S_2).
\]

Thus, we end proof of the Proposition.
\vspace{-5mm}
\end{proof}

With this proposition, we represent the $R$-Union of KP by a merge set $S_1 \cup S_2$. In practice, we express the union of sets by induction (eg. $\{\textit{1990}\} \cup \{\textit{1991}\} \cup, \dots, \cup \{\textit{1999}\}$ as \{$\textit{90s}$\}). Or simply add underlines between them (eg. $\{\textit{1990}\} \cup \{\textit{1991}\} )$ as \{$\textit{1990}\_\textit{1991}$\}). After that, our representation would only have an intersection between triplets.




By introducing variables, our representation naturally handles KP recursion by faltten it into the intersection of KPs. For example, given a recursion $R^{1}(R^{2}(S))$, we can represent it as [[$V@X$, $r_1$, $V@Y$], [$V@Y$, $r_2$, $S$]].

Finally, an IS task $q(T)$ can be represented by a list of triplets.
For example, the question in Eq. (\ref{eq: eg}) can be represented as:

\begin{equation}
\label{eq: eg.r}
\begin{aligned}
        q(T) \triangleq ?T \quad s.t. &\quad [[V@\text{T}, \text{playIn}, V@X], \quad [V@\text{T}, \text{playAt}, C@\text{2004\_05}], \\
        &\quad  [V@\text{T}, \text{bornIn}, C@\text{90s}], \quad [V@X, \text{foundIn}, C@1966], \\
        &\quad [V@X, \text{isA}, C@\text{East German football team}]]
\end{aligned}
\end{equation}

\subsubsection{Layer-wise Expansion Strategy}

After representing the $q(T)$, we now elaborate on the expansion process in each iteration. Expansion strategy is key to our data synthesis. Compared to previous approaches that synthesize or extend questions at the natural language form, our formalization of IS tasks enables systematic analysis of structural question characteristics. This formal framework allows us to explicitly identify latent structural patterns within questions and perform a controlled and optimized expansion paradigm.

\begin{figure*}
    \centering
    \includegraphics[width=1\linewidth]{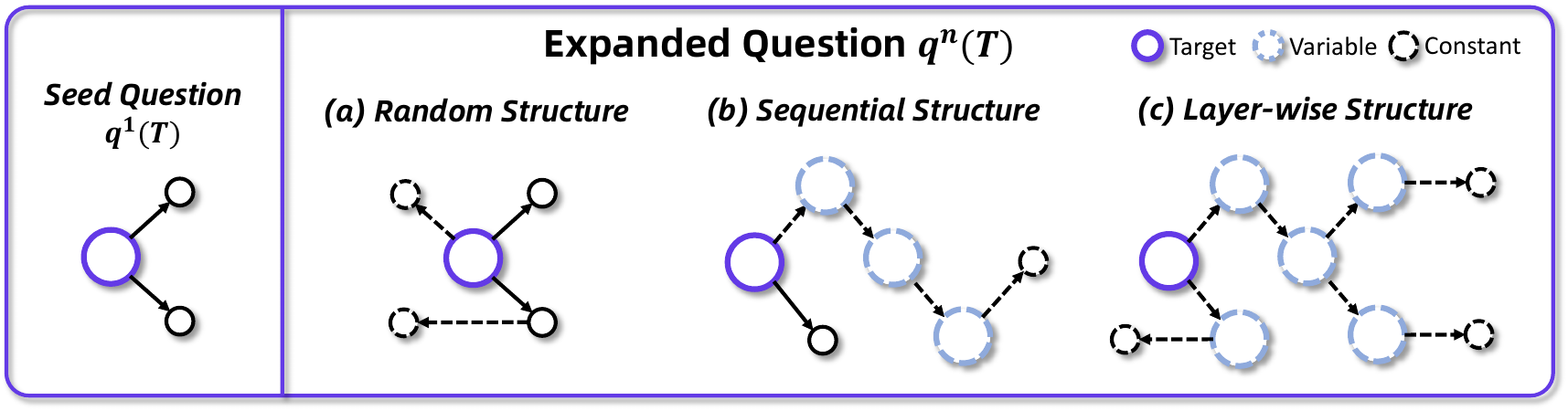}
    \caption{Structures on different expansion paradigms. \textbf{\textit{(a) Random Structure}} denotes expanding by randomly adding constants. 
    \textbf{\textit{(b) Sequential Structure}} is expanding on a chain of reasoning sequence. 
    \textbf{\textit{(c) Layer-wise Structure}} traverses layer-wisely on leaf constants and replaces them with variables.
    ``Target'' stands for target variable. 
    ``Variable'' means the intermediate variable. 
    ``Constant'' is the constant in our KP representation. 
   }
    \label{fig: expansion paradigm}
\end{figure*}

To clearly illustrate the expansion strategy, we show our KP representation in a graph. The nodes in the graph are variables and constants in the list of triplets. And the edges are the relations. For example, the question in Eq.~(\ref{eq: eg.r}) can be illustrated as a graph in Figure~\ref{fig: expansion paradigm}. The question requires determining the target variable via the given constants.

Previous methods are constrained by informal representations of natural language, which limit the controllable expansion and synthesis paradigms for questions. In our formalization language, previous methods would result in question structures as Random~\citep{webwalker, shi2025taskcraft} or Sequential~\citep{wu2025webdancer}. The Random structure stands for methods that directly add FP to any nodes in the graph shown in Figure~\ref{fig: expansion paradigm} (a). Sequential structure is resulted from generating the reasoning chain via a sequence shown in Figure~\ref{fig: expansion paradigm} (b). However, these two paradigms have key limitations:

\begin{itemize}
    \item \textit{Redundancy} As shown in Random structure in Figure~\ref{fig: expansion paradigm}, there exist constants connect to other constants. In this condition, such a sentence as "Dynamo Berlin is a football club based in Berlin" would exist in the question. However, it doesn't increase the reasoning chain of the task-solving.
    
    \item \textit{Reasoning Shortcut} As shown in the Sequential structure in Figure~\ref{fig: expansion paradigm}, there exists an FP which connects constants directly to the target. If this happens, models may guess the answer by only reasoning on the closer constants and neglecting the deeper sequence.
\end{itemize}

To mitigate these limitations, we introduce the Layer-wise Expansion Strategy. We layer-wisely traverse the graph to find all leaf constants. When we obtain all the leaf constants of the current graph, an Expander takes each constant once a time to construct this constant into new FPs. These FPs can form a sub-question that regards the constant as the answer. The expander then merges the sub-question to the current one to form a new one:
\begin{equation}
    q^{n+1}(T)=\text{Expander}(C, q^{n}(T)).
\end{equation}

Note that the $q^{n+1}(T)$ always has the same answer as $q^{n}(T)$. As illustrated in the Figure~\ref{fig: expansion paradigm}, in each expansion, the Expander takes a leaf constant node, turns it into a variable node connected with new nodes. The resulting structure would not have the Redundant and Reasoning Shortcut problems. The number of expanding layers $l$ is a hyperparameter for controlling the task coverage and difficulty.

\subsubsection{Expander Agent}


We now introduce the Expander, an autonomous agent designed to enhance question generation through iterative refinement. 
Given an input constant, the Expander first retrieves relevant contextual information, then formulates a semantically coherent sub-question. 
This sub-question is subsequently integrated with the original query to construct an enriched, context-aware question that better aligns with the underlying information-seeking objective.

The Expander builds upon \texttt{ReAct}~\citep{yao2023react}, a widely-adopted framework for language agents. 
A \texttt{ReAct} trajectory comprises multiple \textcolor{cyan}{\texttt{Thought}}\textbf{-}\textcolor{purple}{\texttt{Action}}\textbf{-}\textcolor{blue}{\texttt{Observation}} interaction cycles. In each cycle, the language model generates free-form \textcolor{cyan}{\texttt{Thought}} for strategic planning, executes structured \textcolor{purple}{\texttt{Action}} to interface with external tools, and receives \textcolor{blue}{\texttt{Observation}} feedback from the environment. 
Formally, the agent execution loop at time $t$ can be represented as $(\tau_t, \alpha_t, o_t)$, where $\tau$ denotes \textcolor{cyan}{\texttt{Thought}}, $\alpha$ signifies \textcolor{purple}{\texttt{Action}}, and $o$ represents \textcolor{blue}{\texttt{Observation}}. 
Each \textcolor{purple}{\texttt{Action}} $\alpha$ decomposes into $(\tau, \phi)$: $\tau$ specifies the action type (using one of the tools or answer), while $\phi$ contains required parameters. 
We equip the Expander with the following tools:

\begin{itemize}
    \item \texttt{Search} This action enables Expander to conduct Google search by severl queries about a constant $c$ and obtains search results. The parameters of this tool are $\phi = \{\textit{queries of}~$c$, \textit{filter\_year}\}$, enabling temporal filtering of search results. This tool would return top relevant URLs and their snippets as \textcolor{blue}{\texttt{Observation}}.
    
    \item \texttt{Summarize} This is the key to $R$-Union oepration. This action allows Expander to visit multiple URLs searched for the constant $c$ and summarize the content. The summarization would integrate the retrieved information to obtain a union constant set as stated in Eq.(\ref{eq: prop}). The parameters of this tool are $\phi = \{\textit{urls}, \textit{goal}\}$. This tool would return the summarization of knowledge about $c$ from the given urls as \textcolor{blue}{\texttt{Observation}}.
    
    \item \texttt{Validate} When Expander completes retrieving and summarizing the KPs of constant $C$, it derives a sub-question and uses this tool to validate the results based on our formalization. The validation purposes are to determine: 1) whether the derived sub-question are consistent with $C$ based on the formalization. 2) whether it is too simple that can be directly answered by an LLM. We call QwQ once time per each purpose. In the first consistency validation, we don't check whether $C$ is strictly the answer to the sub-question. Instead, it checks if the type of $C$ satisfies the sub-question. For the second validation, we require QwQ to answer the sub-question. If the prediction is the same as $C$, we regard it as invalid. This tool would return detailed validation results as \textcolor{blue}{\texttt{Observation}}, and the Expander would take the next action according to it.
    
\end{itemize}

The iterative expansion process terminates upon executing the \textit{answer} action, which finalizes the question construction phase with a verified sub-question derived from the accumulated knowledge.

\subsection{Trajectory Construction}

After harvesting the expanded questions, we proceed to construct task-completing trajectories. 
To this end, we instantiate an agent framework based on QwQ structurally aligned with the Expander, adopting the \texttt{ReAct} paradigm~\citep{yao2023react}.
At each timestep, the agent first first produces a \textcolor{cyan}{\texttt{Thought}} $\tau$ followed by an \textcolor{purple}{\texttt{Action}} $\alpha$. 
It receives the \textcolor{blue}{\texttt{Observation}} $o$ of the \textcolor{purple}{\texttt{Action}} to determine the behavior in the next round.

The agent is equipped with two external tools: \texttt{Search} and \texttt{Visit}.
The \texttt{Search} tool conducts Google search with several queries, which is the same as Expander. \texttt{Visit} returns the pages' information for the given URLs.
For each input question, we perform 5 times rollouts.

To ensure the quality and relevance of the collected trajectories, we further design a set of filtering strategies:
\begin{itemize}
    \item \textit{Correctness} We use a judge LLM to exam the final answer of each trajectory and only keep the correct ones. We also remove if there are tool call errors.
    
    \item \textit{Quality} We filter trajectories if they contain hallucinations of guessing observation and severe repetitions.
    
\end{itemize}

We finally obtain $5,000$ trajectories for later supervised training and reinforcement learning.

\subsection{Agent Training}
To train our information-seeking agent, similar to WebDancer~\citep{wu2025webdancer}, we implement supervised fine-tuning (SFT) followed by reinforcement learning (RL).

In SFT, given a trajectory in a sequence of tokens $\gT=(\tau_1, \alpha_1, o_1, ..., \tau_n, \alpha_n, o_n)$, we mask out loss from observation leading to loss:

\begin{equation}
    \label{eq: sft}
    L =  -\frac{1}{\sum_{i=1}^{|\mathcal{\gT}|} \mathbb{I}[x_i \in o]} \sum_{i=1}^{|\gT|} \mathbb{I}[x_i \in o] \cdot \log \pi_{\theta}(x_i \mid x_{<i})
\end{equation}

where $\pi_{\theta}$ is the model to train. Later in RL, we further optimize $\pi_{\theta}$ use the GRPO algorithm~\citep{shao2024deepseekmath}. 
For a question-answer pair $(q,a)$, GRPO samples rollouts $\{y_i\}_{i}^{|G|}$ and updates the policy model by:

\begin{equation}
\label{eq: grpoloss}
\begin{aligned}
\mathcal{J_{\mathrm{GRPO}}}(\theta) =\quad& \mathbb{E}_{(q,a)\sim \mathcal{D}, \{y_i\}_{i=1}^G\sim \pi_{\theta_\text{old}}(\cdot\mid context)}\\&
\Bigg[\frac{1}{\sum_{i=1}^{G}|y_i|}\sum_{i=1}^{G}\sum_{t=1}^{|y_i|} 
\min \Big( r_{i,t}(\theta) \hat{A}_{i,t},  
\ \text{clip} \Big( r_{i,t}(\theta), 1 - {\varepsilon_{low}}, 1 + {\varepsilon_{high}} \Big) \hat{A}_{i,t} \Big) \Bigg] \\
& r_{i,j}(\theta)=\frac{\pi_\theta \bigl(o_i \mid q_i,\,o_{i,<t}\bigr)}{\pi_{\theta_{\mathrm{old}}}\bigl(o_i \mid q_i,\,o_{i,<t}\bigr)} , \quad \hat{A}_{i,j}= \frac{R_i - \mathrm{mean}\bigl(\{R_i\}\bigr)}{\mathrm{std}\bigl(\{R_i\}\bigr)},
\end{aligned}
\end{equation}

where $context$ includes all the model completions and tool responses.
$\varepsilon$ is the clipping range of the importance
sampling ratio \(r_{i,t}(\theta)\).
$\hat{A}_{i,t}$ is an estimator of the advantage of the $i$-th rollout at $t$-th step.

%% file: sections/3-experiments.tex
\input{tables/main}

\section{Experiments}

\subsection{Experimental Setups}
We evaluate WebShaper on two information-seeking benchmarks: \textbf{GAIA}~\citep{mialon2023gaia} and \textbf{WebWalkerQA}~\citep{webwalker}.
We use the \textit{LLM-as-Judges} paradigm to evaluate both tasks using the \texttt{Pass@1} metric, following~\cite{Li2025webthinker}.

We compare our synthesized dataset with several datasets: 

\begin{itemize}
    \item \texttt{WebWalkerQA} employs random walks over interlinked URLs to synthesize questions based on the visited webpages~\citep{webwalker}. 
    The dataset includes both single-source questions, generated from a single visited URL, and multi-source questions, which are constructed using information aggregated from multiple visited URLs.
    \item \texttt{E2HQA} is a dataset introduced by WebDancer~\citep{wu2025webdancer}, where simple questions are systematically rewritten into more complex, challenging ones.
    \item \texttt{MHQA} is a composite dataset that integrates existing single-hop and multi-hop question-answering datasets. The majority of the questions are annotated by humans.
\end{itemize}


We also compare with cutting-edge deep research methods including Search-o1~\citep{li2025search}, WebWalker~\citep{webwalker}, WebDancer~\citep{wu2025webdancer}, WebThinker~\citep{Li2025webthinker}, SimpleDeepResearch~\citep{sun2025simpledeepsearcher}.

\subsection{Main Results}

We compare WebShaper with cutting-edge baselines. The results are shown in Table~\ref{tab:main_result}. \w~achieves best performances on open-sourced methods on both GAIA and WebWalkerQA. Among all GAIA results, \w-on Qwen-2.5-72B excels second-best method WebSailor 4.7 score. On WebWalkerQA \w~obtains the highest 52.2 score. 

\w~performs the best on each backbone setting. These results indicate the generalizability of the synthesized data on different models. \w~is currently the only open source method with a score of more than 60 points, which is close to the SOTA OpenAI DR system. \w~is implemented fully under open-sourced LLMs, demonstrating that high-quality IS data can deeply stimulate the ability of DR Agents.

\subsection{Discussions}

\subsubsection{Data Statistics}

\begin{wrapfigure}{r}{0.45\textwidth}
\captionsetup{aboveskip=0pt, belowskip=-15pt} 
  \centering
  \includegraphics[width=0.45\textwidth]{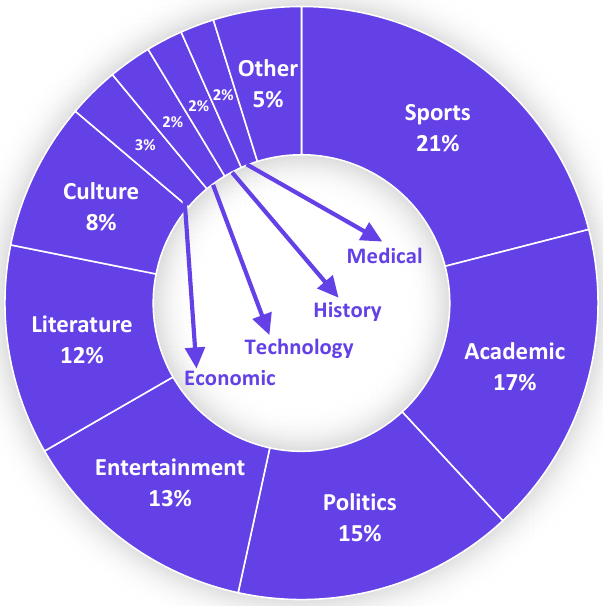} 
  \caption{Domain distribution.}
  \label{fig: domain}
\end{wrapfigure}

We analyze the domain distributions of our dataset. The domain distribution of our dataset demonstrates rather comprehensive coverage across multiple thematic areas, as visualized in Figure \ref{fig: domain}. Our construction of seed tasks leads to questions about various topics and entities. Our agentic expansion further strengthens these benefits. The dataset achieves significant diversity through its balanced representation of major domains such as \texttt{Sports}, \texttt{Politics}, and \texttt{Entertainment}. 

This deliberate design ensures our dataset not only avoids over-reliance on any single domain but also maintains sufficient sample density across diverse topics. The empirical balance between breadth and depth enables robust training of a domain-agnostic information-seeking agent. Such characteristics position our dataset as particularly suitable for train multi-domain IS tasks and fostering interdisciplinary research.

\subsubsection{Data Comparison}

\input{tables/data_comparison}

In this section, we compare \w~with baseline datasets. We sample 5,000 data from each dataset. Then we supervised fine-tune Qwen2.5-32B, Qwen2.5-72B~\citep{qwen2.5}, and QwQ~\citep{qwq} on each dataset. The GAIA results are shown in Table~\ref{tab: data_comparison}.

The comparative results presented in Table~\ref{tab: data_comparison} demonstrate the superior performance of \w~across all backbone architectures on the GAIA benchmarks. Notably, \w~achieves the highest average scores for Qwen-2.5-32B, Qwen-2.5-72B, and QwQ-32B, respectively, significantly outperforming baseline datasets like WebWalkerQA and MHQA.

Even when comparing models with similar parameter counts (e.g., Qwen-2.5-32B), \w-enabled models show substantial improvements.
The consistency of \w's performance improvement suggests its effectiveness in enhancing model capabilities regardless of architectural design.
These findings validate the effectiveness of formalization-driven data synthesis, making it a superior training data solution for information-seeking tasks.

\subsubsection{RL Stimulation}

We compare GAIA performances between models trained after SFT and reinforcement learning. RL models are trained based on the SFT results. As illustrated in Figure \ref{fig:rl_gaia} and \ref{fig:rl_ww}, our experimental results demonstrate significant performance improvements across both Qwen2.5-32B and Qwen2.5-72B models after RL training on both GAIA and WebWalkerQA. The Pass@1 metric shows notable enhancements of +7.8 points for the 32B model and an even more pronounced +13.5 points increase for the 72B variant on GAIA. On WebWalkerQA, \w~also improves IS capability on a large scale. This substantial gain highlights the critical role of RL in activating advanced information-seeking capabilities within LLM.

\begin{figure}[h]
    \centering
    \begin{subfigure}[b]{0.4\textwidth}
        \includegraphics[width=\textwidth]{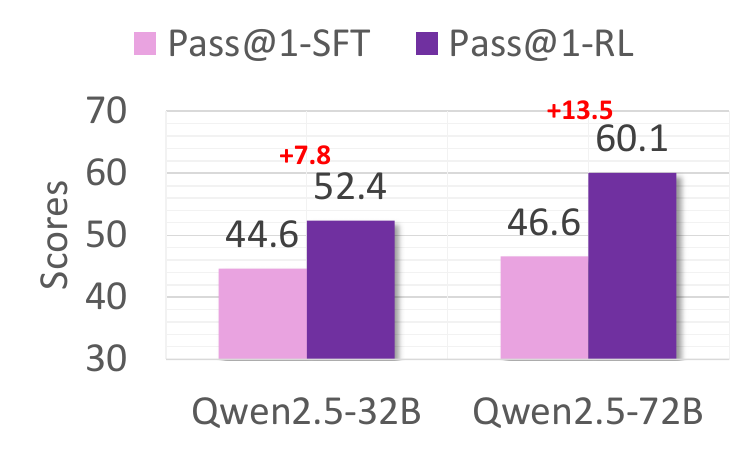}
        \caption{GAIA.}
        \label{fig:rl_gaia}
    \end{subfigure}
    \begin{subfigure}[b]{0.4\textwidth}
        \includegraphics[width=\textwidth]{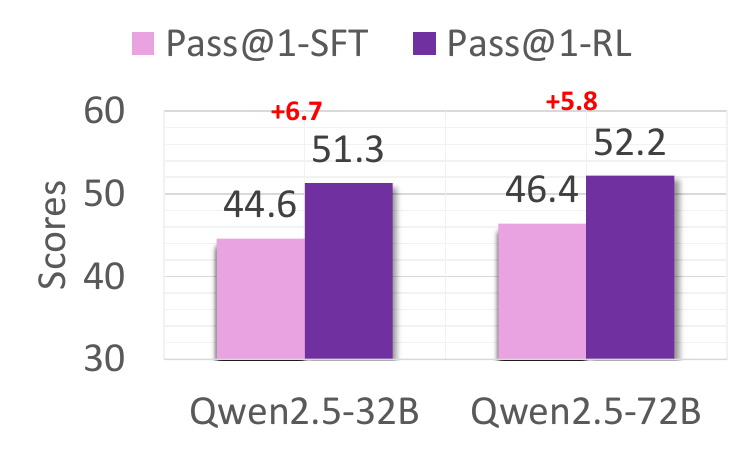}
        \caption{WebWalkerQA.}
        \label{fig:rl_ww}
    \end{subfigure}
    \label{fig: rl}
    \caption{Comparison with SFT and RL.}
\end{figure}

The breadth and complexity of tasks introduced by our task formalization stimulate dynamic IS strategies during RL. Unlike generic datasets, our carefully curated scenarios require the model to iteratively query relevant information, effectively "training" it to prioritize contextually aligned knowledge fragments.

\subsubsection{Formalization}

In this part, we validate whether our formalization truly improves the dataset. We compare our dataset to a variation that uses natural language during the data synthesis. This variation takes the current question in each iteration and also uses the Expander agent to expand it to a new question. The Expander process in natural language as well. We SFT Qwen2.5-32B, Qwen2.5-72B, and QwQ on both datasets. The other training setting remains the same. We compare the training results with the variation as shown in Figure~\ref{fig:dis_fl}.

FL excels NL in all base model backbones. These results indicate that our formalization language can mitigate the limitations incurred by natural language. Our IS task formalization can synthesize more forms of tasks. It also reduces error propagation in the synthesis process, leading to consistent and precise question-and-answer pairs.

\subsubsection{Layer-wise Expansion Strategy}

We evaluate the effectiveness of the Layer-wise structure. In order to compare, we set up a variation which uses the same Expander and task formalization but expands the question in a sequence as shown in Figure~\ref{fig: expansion paradigm}. We SFT Qwen2.5-32B, Qwen2.5-72B, and QwQ on both datasets. Other training settings remain the same. The results as shown in Figure~\ref{fig:dis_lw}.

The layer-wise structure performs better than the Sequential structure in all base models. The results show that our method truly mitigates shortcomings such as Redundancy and Reasoning shortcuts. Our method improves the final performance via the controllable structures.

\subsubsection{Tool Call Analysis}

\begin{figure}
    \centering
    \begin{subfigure}[b]{0.47\textwidth}
        \includegraphics[width=1\textwidth]{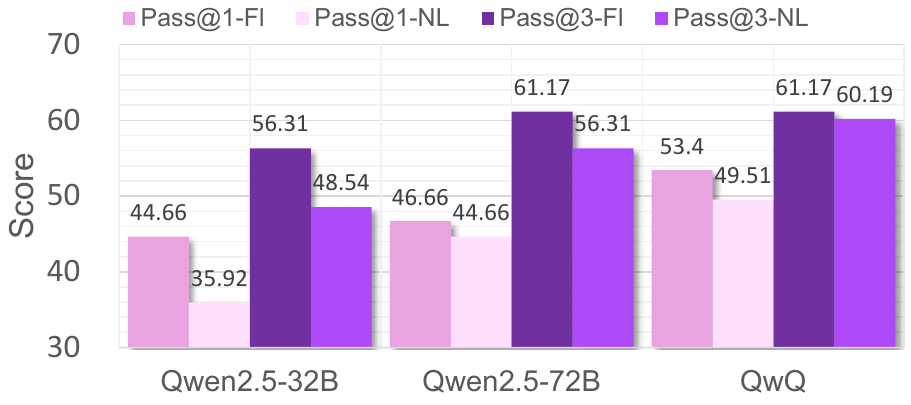}
        \caption{Formalization ablation analysis.}
        \label{fig:dis_fl}
    \end{subfigure}
    \hfill
    \begin{subfigure}[b]{0.47\textwidth}
        \includegraphics[width=1\textwidth]{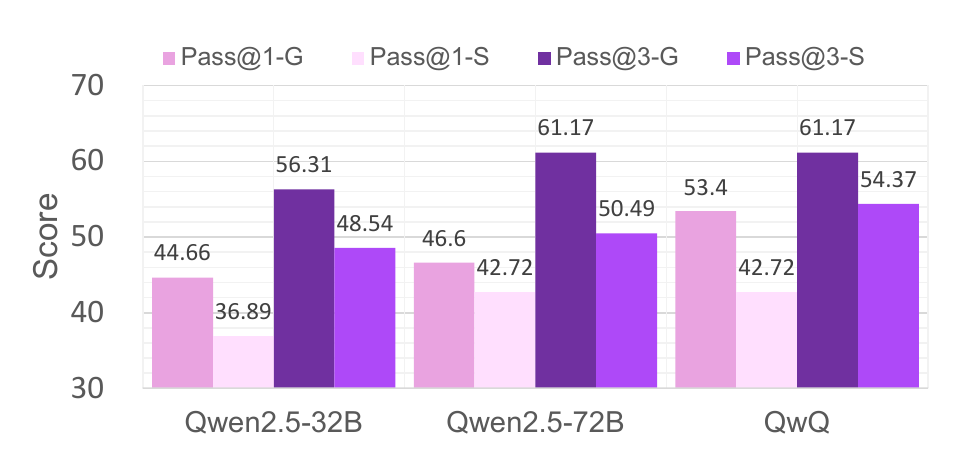}
        \caption{Layer-wise structure ablation analysis.}
        \label{fig:dis_lw}
    \end{subfigure}
    \label{fig: disscusion}
    \caption{Discussions on formalization and layer-wise structure.}
\end{figure}

\begin{figure}[htbp]
    \centering
    \begin{subfigure}[b]{0.31\textwidth}
        \includegraphics[width=\textwidth]{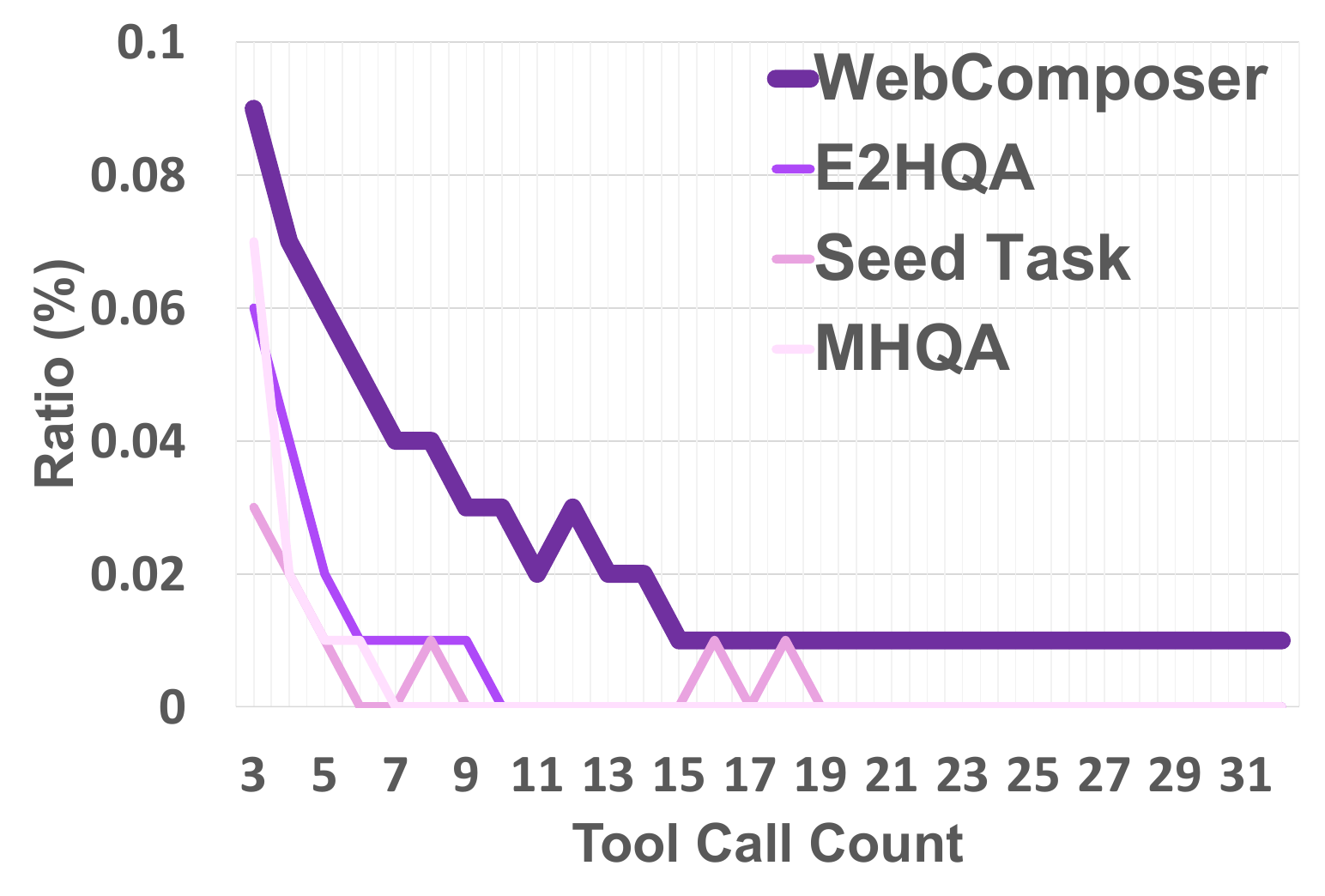}
        \caption{Search distribution.}
        \label{fig:search}
    \end{subfigure}
    \hfill
    \begin{subfigure}[b]{0.31\textwidth}
        \includegraphics[width=\textwidth]{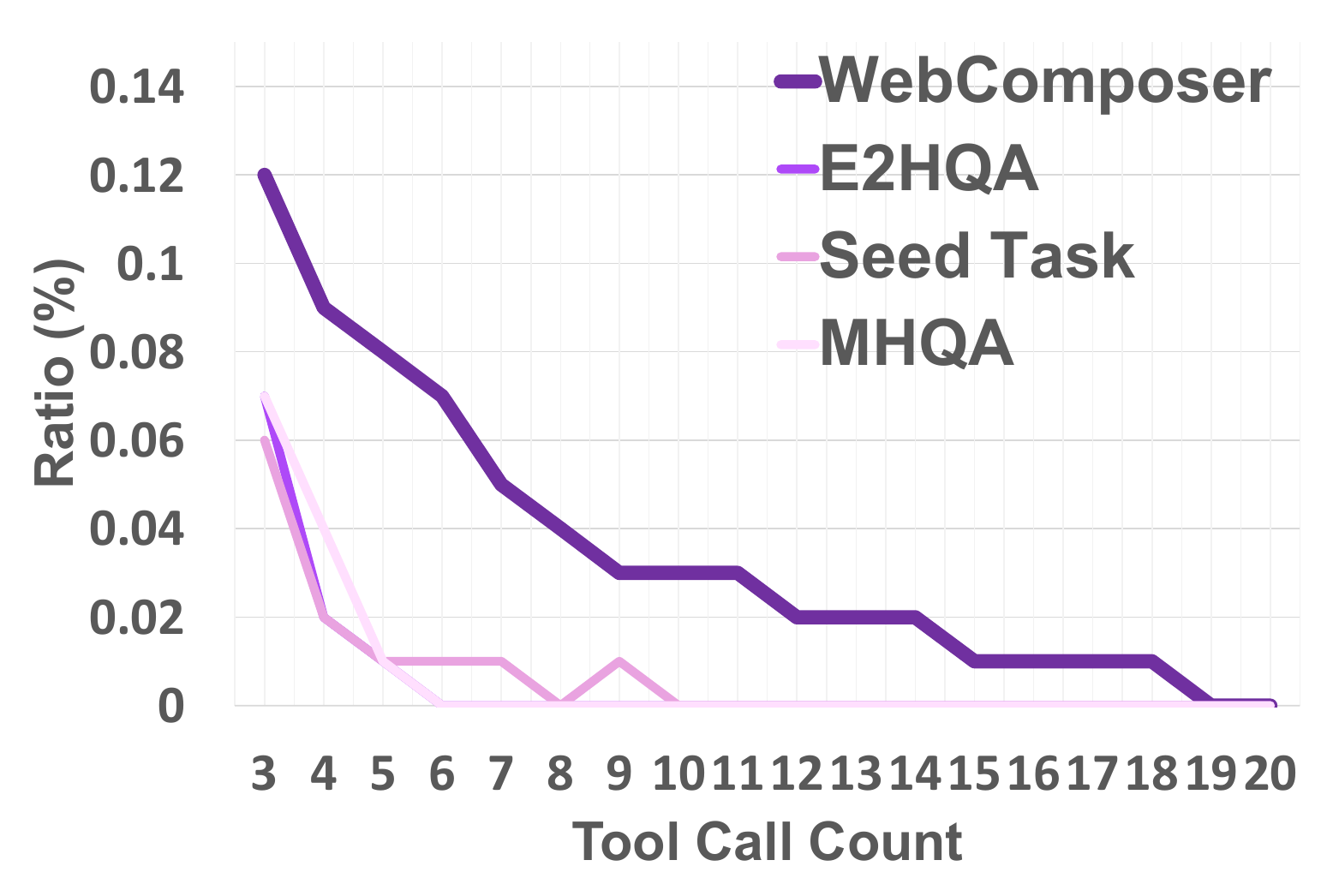}
        \caption{Visit distribution.}
        \label{fig:visit}
    \end{subfigure}
    \hfill
    \begin{subfigure}[b]{0.31\textwidth}
        \includegraphics[width=\textwidth]{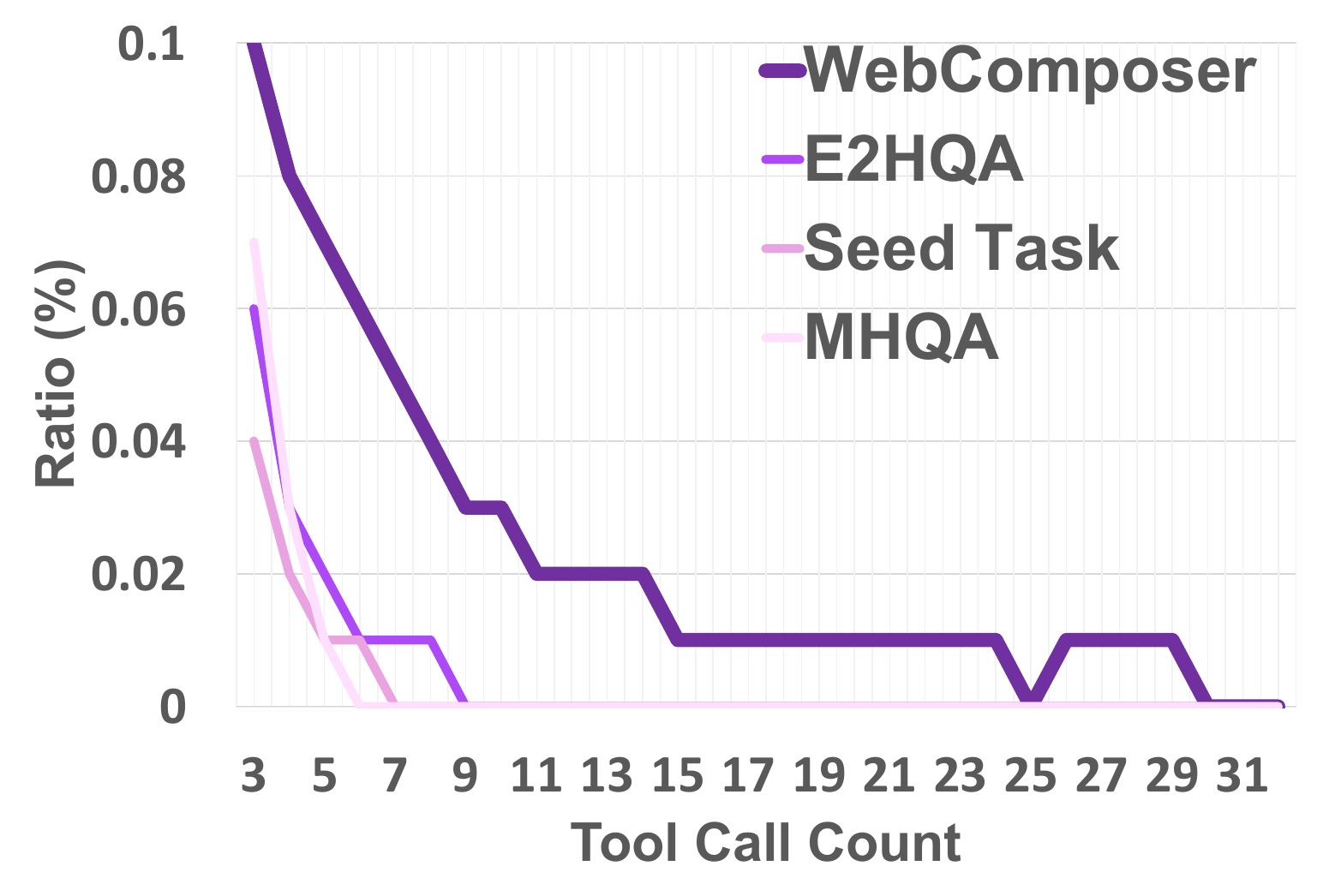}
        \caption{Total tool distribution.}
        \label{fig:tc}
    \end{subfigure}
    \caption{Tool call analysis.}
    \label{fig:three}
\end{figure}

We show the distribution tool call count of the agent to solve a question in different datasets. We illustrate the tool call counts larger than 3, which shows the complicated trajectories proportion.

\textbf{Search Complexity (Figure~\ref{fig:search}})
\w~ exhibits a pronounced long-tail distribution. Pretty much tasks requiring over 3 search operations. This is 3-4x higher than E2HQA and MHQA, indicating superior handling of information-rich queries requiring iterative refinement.

\textbf{Knowledge Navigation (Figure~\ref{fig:visit}})
The visit operation distribution shows \w~ maintains a high ratio for trajectories exceeding 3 steps, while competing datasets sharply drop after 10 steps. This sustained capability reflects enhanced navigational intelligence in IS tasks.

\textbf{Composite Reasoning (Figure~\ref{fig:tc}})
In total tool calls, \w's doubles the count larger than 3. Notably, it sustains non-zero proportions up to 30 tool calls, demonstrating scalability for highly complex compositional reasoning.

These findings underscore \w's unique ability to manage intricate reasoning chains, with statistically significantly higher proportions of multi-hop reasoning trajectories across all modalities. The sustained performance in extended tool call sequences suggests superior architectural capacity for managing complex task decompositions compared to existing benchmarks.

\subsubsection{Case Study}

\begin{figure*}[h]
    \centering
    \includegraphics[width=1\linewidth]{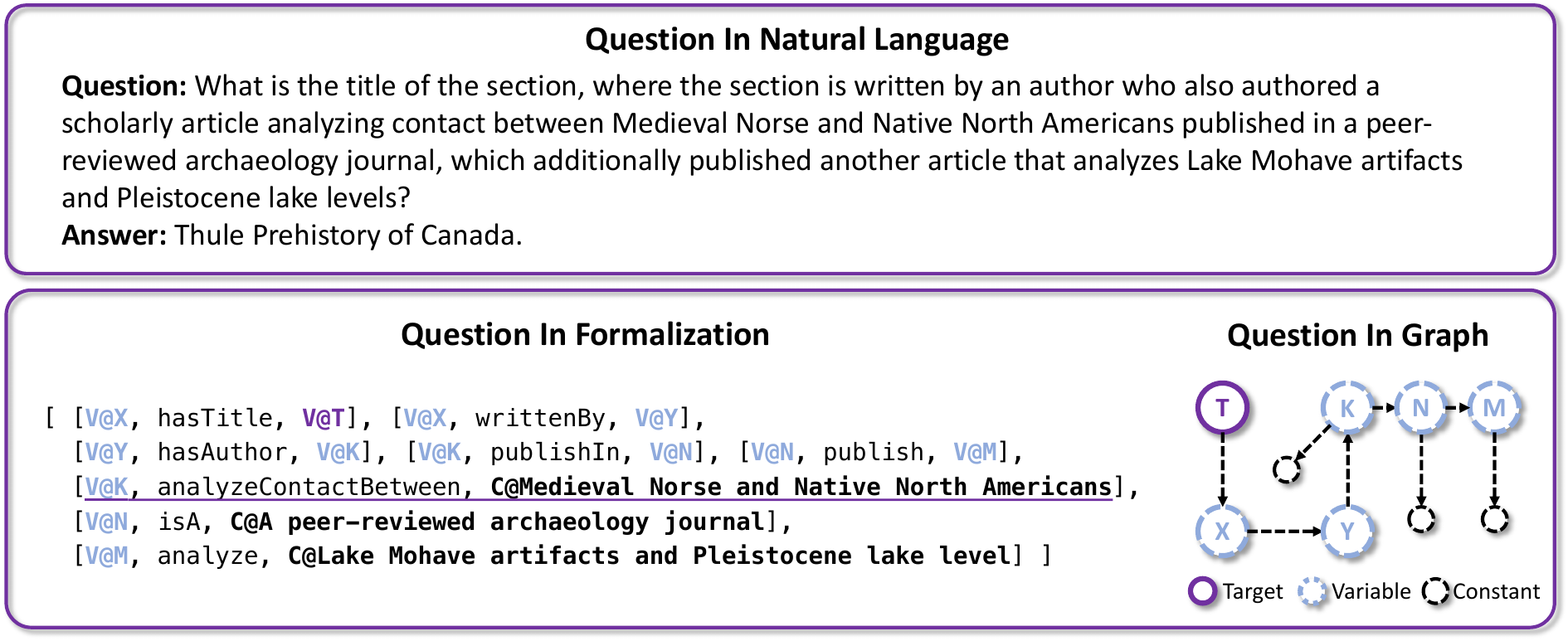}
    \caption{Case studies of our synthesized data. We show a question in natural language, our formalization, and a graph respectively.}
    \label{fig: case}
\end{figure*}


We present a representative case study in Figure~\ref{fig: case}. Compared with linear structure and sequential structure, our synthesized data has no problems of redundancy and reasoning shortcuts. The model should strictly seek information and reason alongside all the variables to find the answer. There are no constants directly connected to the target variable $T$ or variables close to it. Besides, there are no constants connected to other constants. We show more cases in the Appendix \ref{app: case}.

Moreover, $R$-Union effects well in our data. The underlined FP is a summarization of distributed web contents, leading to more difficulty in resolving the variables $K$, $N$, and $M$. Benefiting from the formalization, our data contains a variety of IS forms, which can fully stimulate the different IS capabilities of the model.

%% file: tables/main.tex
\begin{table}[h]
\small
\centering
\caption{\textbf{Main results} on GAIA and WebWalkerQA benchmarks.
We compare \w~with several cutting-edge baselines methods.
\textbf{bolded} number stands for the best results on the corresponding settings. 
\textcolor{blue}{Blue} scores are the highest among all open-sourced methods.
}
\resizebox{\columnwidth}{!}{%
\begin{tabular}{@{}lc|cccc|cccc@{}}
\toprule
& & \multicolumn{4}{c|}{\textbf{GAIA}}        & \multicolumn{4}{c}{\textbf{WebWalkerQA}} 
            \\ \midrule
\textbf{Backbone} &\textbf{Framework}     & \texttt{Level 1}            & \texttt{Level 2}  & \texttt{Level 3}    & \texttt{Avg.}   &\texttt{Easy}  & \texttt{Medium}  &\texttt{Hard} & \texttt{Avg.} \\ 
\midrule
\multicolumn{10}{c}{\cellcolor{blue!10} \textbf{\textit{No Agency}}} \\
\arrayrulecolor{black}\midrule
Qwen-2.5-7B & Base& 12.8 & 3.8 & 0.0 & 6.8 & 1.25 & 0.8 & 0.7 & 0.8\\
\arrayrulecolor{black!20}\midrule
\multirow{2}{*}{Qwen-2.5-32B} & Base& 20.5 & 9.6 & 8.3 & 13.6 & 3.8 & 2.5 & 3.3 & 3.1\\
 & RAG & 12.8 & 11.8 & 8.3 & 11.8 & 23.1 & 14.3 & 11.3 & 15.3\\
 \arrayrulecolor{black!20}\midrule
Qwen-2.5-72B  & Base& 20.5 & 13.5 & 0.0 & 14.6 &9.4&7.1&3.3&6.3\\
\arrayrulecolor{black!20}\midrule
GPT-4o  & Base& 23.1&15.4&8.3&17.5&6.7&6.0&4.2&5.5\\
\arrayrulecolor{black!20}\midrule
\multirow{2}{*}{QwQ-32B}  & Base& 30.8 & 15.4 & 25.0 & 22.3 &7.5&2.1&4.6&4.3\\
& RAG & 33.3 & 36.5 & 8.3 & 32.0 & 36.9 & 26.1 & 33.5 & 31.2\\
\arrayrulecolor{black!20}\midrule
DeepSeek-R1-671B & Base& 43.6 & 26.9 & 8.3 & 31.1 & 5.0 & 11.8 & 11.3 & 10.0\\
\arrayrulecolor{black}\midrule
\multicolumn{10}{c}{\cellcolor{blue!20}\textbf{\textit{Close-Sourced Agentic Frameworks}}}\\
\midrule
 & \textit{OpenAI DR} & \textcolor{gray}{74.3} & \textcolor{gray}{69.1} & \textcolor{gray}{47.6} & \textcolor{gray}{67.4} & - & - & - & -  \\
\midrule
\multicolumn{10}{c}{\cellcolor{blue!25}\textbf{\textit{Open-sourced
 Agentic Frameworks}}}\\
\midrule


\multirow{3}{*}{Qwen-2.5-32B} & Search-o1 & 33.3 & 25.0 & 0.0 & 28.2 & - & -  &- & - \\ 
& WebDancer & 46.1 & 44.2 & 8.3 & 40.7  & 44.3 & 46.7  & 29.2 &  38.4 \\ 
& \textbf{\w} &61.5 & 53.8& 16.6 & \textbf{52.4} & 58.1 & 51.4 & 47.0 & \textbf{51.4} \\ 
\arrayrulecolor{black!20}\midrule
\multirow{6}{*}{QwQ-32B} & Search-o1 & 53.8 & 34.6 & 16.6 & 39.8 & 43.1 & 35.0 & 27.1 & 34.1 \\ 
& WebThinker-Base & 53.8 & 44.2 & 16.6 &44.7 & 47.2 & 41.1 & 39.2 & 41.9  \\ 
& WebThinker-RL & 56.4 & 50.0 & 16.6 & 48.5 & 58.8 & 44.6 &40.4&46.5   \\ 
& Simple DS & - & - & - & 50.5 & - & -& -  & - \\ 
& WebDancer & 61.5 &  50.0 & 25.0 & 51.5 & 52.5 &  59.6 & 35.4  & 47.9 \\ 
& \textbf{\w} & 69.2 & 50.0 & 16.6 & \textbf{53.3} & 55.8 & 49.2 & 45.4 & \textbf{49.7}  \\ 
\arrayrulecolor{black}\midrule
 \multirow{1}{*}{Qwen-2.5-72B}& \textbf{\w} & 69.2 & 63.4 & 16.6 & \textcolor{blue}{\textbf{60.1}} & 56.2 & 52.1 & 49.5 & \textcolor{blue}{\textbf{52.2}} \\ 

\arrayrulecolor{black}\bottomrule
\end{tabular}
}
\label{tab:main_result}
\end{table}

%% file: tables/data_comparison.tex
\begin{table}
\small
\centering
\caption{\textbf{SFT Data Comparison} on GAIA benchmarks.
The best results among all backbones are in \textbf{bolded}.
}

\setlength{\tabcolsep}{10pt} 
\begin{tabular}{@{}lc|cccc@{}}
\toprule
& & \multicolumn{4}{c}{\textbf{GAIA}} \\ \midrule
\textbf{Backbone} &\textbf{Dataset}     & \texttt{Level 1}            & \texttt{Level 2}  & \texttt{Level 3}    & \texttt{Avg.}  \\ 
\midrule

\multirow{4}{*}{Qwen-2.5-32B} & WebWalkerQA & 43.5 & 30.7 & 0.0 & 32.0  \\ 
& E2HQA & 56.4 & 36.5 & 0.0 & 39.8 \\ 
& MHQA & 43.5 & 36.5 & 8.3 & 35.9 \\ 
& \textbf{\w} & 56.4 & 40.3 & 16.6 & \textbf{43.6}  \\ 

\arrayrulecolor{black!20}\midrule

\multirow{4}{*}{Qwen-2.5-72B} & WebWalkerQA & 53.8 &  36.5 & 0.0 & 38.8 \\ 
& E2HQA & 61.5 & 38.4 & 16.6 & 44.6 \\ 
& MHQA & 56.4 & 44.2 & 0.0 & 43.6 \\ 
& \textbf{\w} & 56.4 & 48.0 & 0.0 & \textbf{45.6} \\ 

\arrayrulecolor{black!20}\midrule

\multirow{4}{*}{QwQ-32B} & WebWalkerQA & 66.6 & 38.4 & 8.3 & 45.6  \\ 
& E2HQA & 58.9 & 42.3 & 16.6 & 45.6  \\ 
& MHQA & 51.2 & 44.2 & 0.0 & 41.7   \\ 
& \textbf{\w} & 69.2 & 50.0 & 16.6 & \textbf{53.3}  \\ 

\arrayrulecolor{black}\bottomrule
\end{tabular}

\label{tab: data_comparison}
\end{table}

%% file: sections/4-related-work.tex
\section{Related Work}


\subsection{Information-Seeking Data Synthesis}
Recent advances in information-seeking agents aim to integrate web interaction into LLMs' reasoning~\citep{Li2025webthinker,song2025r1,jin2025search,shi2025pangu,chen2025research,zhang2025evolvesearch,wu2025masksearchuniversalpretrainingframework}.
While these works exhibit promising capabilities, they predominantly depend on limited or overly simplistic datasets~\citep{yang2018hotpotqa,joshi2017triviaqa,kwiatkowski2019natural}.
Concurrently, several recent benchmarks, such as GAIA~\citep{mialon2023gaia}, BrowseComp~\citep{bc_en}, and BrowseComp-zh~\citep{bc_zh}, provide only test sets, which restricts their applicability for training agents.
Early efforts, such as WebWalkerQA~\citep{webwalker}, explored simulating human-like web navigation to generate QA pairs by constructing linear information chains.
\textsc{crawl}QA within WebDancer~\citep{wu2025webdancer} expands simple questions to more complex ones by aggregating external information, while SailorFog-QA within WebSailor~\citep{li2025websailor} leverages entity coreference networks to support fuzzy reasoning. 
These methods are predominantly information-driven, focusing on strategies for retrieving and connecting knowledge.
In contrast, our approach is formalization-driven, emphasizing the structural representation and principled modeling of the QA process.

\subsection{Formalization-based Data Synthesis}

Formalization-based data synthesis is common in the study of theorem proving in LLM mathematics.
DeepSeek-MathProver synthesizes data to train a math theorem prover. It transforms high school and undergraduate level math competition problems into formal statements. It then automatically generates proofs by an LLM and verify the correctness of these proofs in a Lean 4 environment~\citep{xin2024deepseek}.
After that, DeepSeek-MathProverV2 decomposes the proof into subgoals. Then synthesis training data to train a small model for the subgoal proof in formal statements~\citep{ren2025deepseek}. 
\cite{leang2025theorem} synthesizes the training data of Theorem Prover as a Judge based on mathematical formalization. Each question needs to go through multiple formal language and natural language conversion and verification processes to ensure the validity of the data. They trained the judger on the synthetic data, and then used the judger to replace the human evaluation in RLHF~\citep{ouyang2022training}, improving the effect of DPO~\cite{rafailov2023direct}.
Goedel-Prover trains LLMs to convert natural language math problems to formal statements in Lean 4. Next, it creates a large dataset of formal proofs by training a series of provers, where each new prover can prove statements that could not be proved by previous ones~\citep{lin2025goedel}.
Another group of related studies is synthesizing training data for knowledge base question answering. These methods formalize the KBQA question via propositional logic.
LACT constructs the arbitrary first-order logical queries similar to \citet{choudhary2023complex} via binary
tree decomposition~\citep{xia2025improving}. This results in an SFT dataset. It then fine-tunes on an easy-to-hard curriculum to stimulate the reasoning capability of LLMs. 
Rather than proposition logics, our work establishes IS formalization via set theory.

%% file: sections/5-conculusion.tex
\section{Conclusion}

This work presents a paradigm-shifting framework for synthesizing training data \w~for information-seeking (IS) agents through formalization-driven design. By establishing a set theory-based mathematical formalization of IS tasks, we address critical limitations in existing information-driven approaches that suffer from structural inconsistencies, task controllability, diversity, and coverage. The composition of proposed Knowledge Projections enables precise engineering of reasoning structures and complexity. Our agentic Expander module further ensures systematic expansion of formalized tasks with a layer-wise expansion paradigm, combining autonomous knowledge retrieval and rigorous validation to minimize redundancy and prevent reasoning shortcuts.  
Experimental results demonstrate that \w~not only achieves state-of-the-art performance on GAIA and WebWalkerQA benchmarks but also introduces controllability over task design, enabling deliberate engineering of cognitive challenges for IS agents. This formalization-driven paradigm shifts the focus from reactive information organization to proactive task specification, opening new avenues for advancing agent capabilities.

%% file: sections/6-appendix.tex
\section{Agent Details}
Following~\cite{wu2025webdancer}, WebComposer uses two tools, \textit{search} and \textit{visit},
which are regarded as fundamental to the information seeking process~\citep{zhu2025oagents}:
\begin{itemize}
    \item \textbf{\textit{Search}} interfaces with the Google search engine to retrieve relevant documents given natural language queries. It supports multiple queries in parallel and returns the top-10 results for each query, where each result includes a title, a snippet, and the corresponding URL.
    \item \textbf{Visit} enables targeted extraction from specific web pages. Each page is paired with a designated visit goal. The full content of the page is first retrieved using Jina~\citep{jina}, after which a summarization model (Qwen-2.5-72B in our implementation) extracts information relevant to the specified goal.
\end{itemize}

\section{Training Details}

\subsection{SFT}
For SFT, we use a batch size of 32 and a learning rate of 5e-6, warmup plus cosine decay schedule.
We also apply a weight decay of 0.1.
\subsection{RL}
For RL training~\citep{sheng2025hybridflow}, each group consists of 8 rollouts.
The temperature is 1.0, $top_p=1.0$, the batch size is 128, the mini batch size is 32, and the learning rate is 1e-6.

\section{Case Study}
\label{app: case}

\begin{figure*}[h]
    \centering
    \includegraphics[width=1\linewidth]{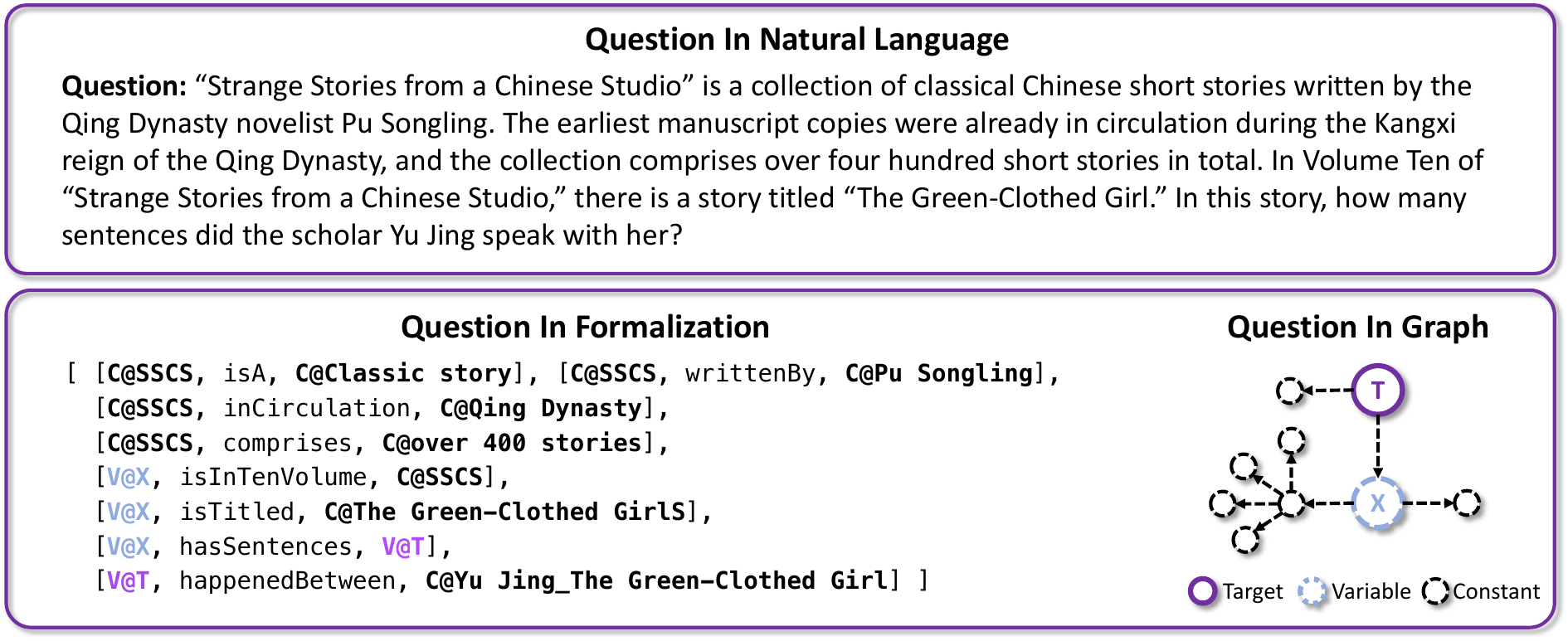}
    \caption{Case comparison. ``SSCS'' stands for "Strange Stories from a Chinese Studio".}
    \label{fig: kimi_case}
\end{figure*}

We compare a representative example shown by KIMI-Researcher~\citep{kimidr}, illustrated in Figure~\ref{fig: kimi_case}. 
The case includes redundant information, such as multiple constants connected to ``SSCS'', which contribute little to answering the question. Additionally, a reasoning shortcut is observed that directly connects to the target variable. Despite the apparent complexity, the underlying reasoning structure is relatively simple, consisting of a single-hop reasoning path.

\section{Broader Impact}

Our data synthesis framework presents a foundational methodology for constructing training data for intelligent agents, featuring two key innovations: \textbf{task formalization} and \textbf{agent-driven synthesis}. By explicitly modeling tasks as structured, formal representations and leveraging proxy agents to synthesize data, this work provides a systematic approach to address the critical challenge of generating training data that transcends the complexity and unpredictability of naturally occurring human-centric environments. Below, we discuss the broader implications for agent research.  

\paragraph{Implications in Agent Training Data Synthesis} 
Traditional approaches to training agents often rely on datasets derived from human-generated interactions, which are inherently limited in diversity, scalability, and controllability. We emphasize that effective agent training requires \textbf{explicit formalization of task structures}—a prerequisite for achieving precise control over data properties. By decoupling task definitions from data generation, the framework enables:

\begin{itemize}
    \item \textit{Targeted Complexity Management}: Tasks can be systematically parameterized to adjust difficulty, modality, or compositional structure, ensuring agents are exposed to controlled gradients of challenge. This contrasts with ad-hoc methods that risk overfitting to biases in natural data or failing to stress-test edge cases.  
    \item \textit{Quality Assurance}: Formal task models act as a "specification" for data synthesis, reducing noise and ensuring consistency. This is critical for applications where reliability and safety are paramount, such as autonomous systems or medical AI.  
    \item \textit{Scalable Data Generation}: Agent-driven synthesis eliminates the need for laborious manual annotation or heuristic-based pipelines by directly translating formal task representations into training instances. This reduces computational overhead while preserving fidelity to the task’s intended design. 
\end{itemize}

\paragraph{Implications for AI Research and Development}
Our architecture provides insights for advancing AI systems:  
\begin{itemize}

\item \textit{Beyond Human-Level Complexity}: By formalizing tasks independent of human behavioral priors, the framework enables training data to exceed the implicit constraints of natural data. This opens pathways to train agents for domains requiring superhuman reasoning (e.g., advanced scientific modeling, combinatorial optimization).   

\item \textit{Cross-Domain/Task Generalization}: Formal task representations abstract away domain-specific noise, allowing agents to learn invariant principles applicable across diverse contexts.   
\end{itemize}